\documentclass{article}

 \usepackage[nonatbib, preprint]{neurips_2022}

\usepackage[utf8]{inputenc} % allow utf-8 input
\usepackage[T1]{fontenc}    % use 8-bit T1 fonts
\usepackage{amsmath,amsfonts,amsthm,amssymb,fullpage,enumerate}
\usepackage{pgf,tikz,tikz-cd}
\usetikzlibrary{positioning,shapes}
\usepackage{graphicx}
\usepackage{multirow}
\usepackage{array}
\usepackage{setspace}
\usepackage{listings}
\usepackage{caption}
\usepackage{subcaption}
\usepackage{epstopdf}
\usepackage{mathrsfs}
\usepackage{cite}
\usepackage{hyperref,url}
\usepackage{algorithm}
\usepackage{ulem}
\usepackage[noend]{algpseudocode}
\usepackage{MnSymbol}

\newcommand{\RR}{\mathbb{R}}

\newcommand{\EE}{\mathbb{E}}

\newcommand{\PP}{\mathbb{P}}

\newcommand{\iid}{\overset{iid}{\sim}}

\newtheorem{thm}{Theorem}

\newtheorem{prop}[thm]{Proposition}
\newtheorem{eg}[thm]{Example}

\newtheorem{Lem}[thm]{Lemma}

\newcommand{\bds}{\begin{displaystyle}}
\newcommand{\eds}{\end{displaystyle}}
\newcommand{\bpm}{\begin{pmatrix}}
\newcommand{\epm}{\end{pmatrix}}
\newcommand{\bvm}{\begin{vmatrix}}
\newcommand{\evm}{\end{vmatrix}}

\usepackage{xcolor}

\def\cred{\color{red}}

\usepackage[utf8]{inputenc} % allow utf-8 input
\usepackage[T1]{fontenc}    % use 8-bit T1 fonts
\usepackage{hyperref}       % hyperlinks
\usepackage{url}            % simple URL typesetting
\usepackage{booktabs}       % professional-quality tables
\usepackage{amsfonts}       % blackboard math symbols
\usepackage{nicefrac}       % compact symbols for 1/2, etc.
\usepackage{microtype}      % microtypography

\title{Testing for geometric invariance and equivariance}

\author{
  Louis G. Christie\\
  Statistical Laboratory, DPMMS\\
  University of Cambridge\\
  Cambridge, UK, CB3 0WB \\
  \texttt{lgc26@cam.ac.uk} \\
  \And
  John A. D. Aston \\
  Statistical Laboratory, DPMMS\\
  University of Cambridge\\
  Cambridge, UK, CB3 0WB \\
  \texttt{j.aston@statslab.cam.ac.uk} 
}

\begin{document}

\maketitle

\begin{abstract}
  Invariant and equivariant models incorporate the symmetry of an object to be estimated (here non-parametric regression functions $f : \mathcal{X} \rightarrow \mathbb{R}$). These models perform better (with respect to $L^2$ loss) \cite{lyle2020benefits}, \cite{bietti2021sample}, and are increasingly being used in practice  \cite{bronstein2021geometric}, but encounter problems when the symmetry is falsely assumed. In this paper we present a framework for testing for $G$-equivariance for any semi-group $G$. This will give confidence to the use of such models when the symmetry is not known a priori. These tests are independent of the model and are computationally quick, so can be easily used before model fitting to test their validity. 
\end{abstract} 
 
\section{Introduction}
Many objects we wish to model obey symmetries. A particularly simple example of this is in time series: if our series is seasonal then it is invariant to translations by its period. In imagery, true classifications are invariant under reflections and rotations of the image. In chemoinformatics molecules are invariant under certain permutations of the atoms.

There have been many successful methods used to incorporate these symmetries in models, particularly data augmentation \cite{chen2020group} and feature averaging \cite{lyle2020benefits}. Recently, Bronstein et al. have formulated the notion of \textit{Geometric Deep Learning} \cite{bronstein2021geometric} and demonstrated the applicability of these symmetries to numerous network architectures: symmetries of the symmetric group $S_n$ for Graph Neural Networks (GNNs), $SO(3)$ rotational symmetries for Spherical Convolutional Neural Networks, and translation group $\mathbb{T}$ invariance for standard Convolutional Neural Networks, for examples. 

Modelling using these symmetries is a way to expose and exploit regularities in the structure of the object of interest. In many cases (as with those above) we know that these symmetries are present, but in other cases we do not. One example of this is in protein volume estimation, where some proteins exhibit rotational symmetries of some unknown order or about some unknown axis \cite{jiang2017atomic}. We would like to be able to use the machinery of equivariant modelling to these situations, but need confidence that it is worth building the symmetrised model. 

The problem is that if the symmetry does not exist, then the model that enforces it converges to the wrong result, as shown in figure \ref{fig:non_inv_error}, and it can do this even with lower test error in some cases. Thus in this paper we present a testing framework for the presence of these symmetries in regression functions, and classifications functions through probabilistic regression models. We present two tests, one that uses an explicit bound on the possible variations if the symmetry were present, and another that uses more computational power to estimate these possible variations. Both tests are shown to work well in low dimensional simulations and can correctly identify a lack of reflective symmetries in the MNIST dataset (i.e., identifying $3 \neq \mathcal{E}$).

Importantly, these tests can be computed before training the equivariant model. Thus they can be used as a validation tool before spending large computational time (in back propagation) or human time (in model tuning). In this sense they can form an important part of the machine learning pipeline for the situations where the symmetry is not known but may be applicable.

\section{Background and our contributions}

Mathematically, we can describe the problem as follows. Consider learning a regression function $f : \mathcal{X} \rightarrow \mathcal{Y}$ with some (perhaps noisy) i.i.d. data $\mathcal{D} = \{ (X_i, Y_i) \}_{i = 1}^n $ with $\EE( Y_i \mid X_i = x) = f(x)$. In this paper $\mathcal{X}$ can be any metric space and $\mathcal{Y}$ can be any normed linear space. Let $G$ be a semi-group that acts on both $\mathcal{X}$ and $\mathcal{Y}$, written $g \cdot x$ and $g \star y$. We say that a regression function $f : \mathcal{X} \rightarrow \mathcal{Y}$ is \textbf{$G$-invariant} if $\PP( f(g \cdot X_1) = f(X_1) ) = 1$  for all $g \in G$. We call $f$ \textbf{$G$-equivariant} if $\PP(  f(g \cdot X_1) = g \star f(X_1) ) = 1$, noting that invariance is a special case where $G$ acts trivially on $\mathcal{Y}$. Examples of an invariant classification function and an equivariant clipping mask on images are shown in figure \ref{fig:in_eq_eg}. 
 
\begin{figure}[h]
	\centering
		\begin{tabular}{m{7em} m{2em} m{1em} m{6em} m{7em} m{2em} m{7em}}
				\includegraphics[scale=0.5]{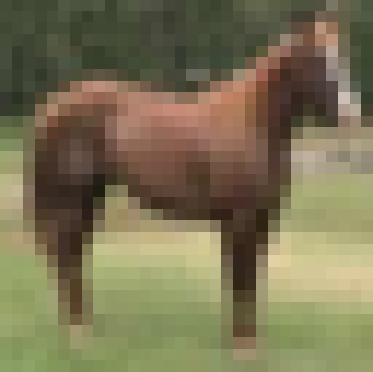} & $\overset{f_c}{\longrightarrow}$ & ``Horse'' &  &
					\includegraphics[scale=0.5]{figs/im5.png} & $\overset{f_m}{\longrightarrow}$ & \includegraphics[scale=0.8]{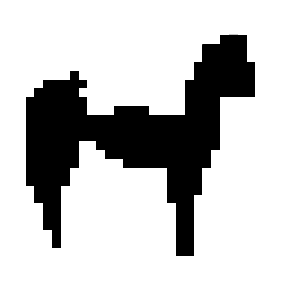} \\

				\includegraphics[scale=0.5]{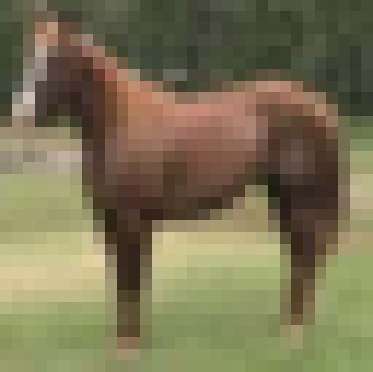} & $\overset{f_c}{\longrightarrow}$ & ``Horse'' &  &
					\includegraphics[scale=0.5]{figs/im5_ref.png} & $\overset{f_m}{\longrightarrow}$ & \includegraphics[scale=0.8]{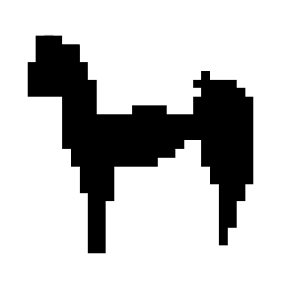}
		\end{tabular}	
		
	\caption{Example images of a horse from CIFAR-10, and the action of the horizontal reflection $R_h$. The function $f_c$ assigns classifications of a horse in the image, which is assumed to be invariant to reflections. The function $f_m$ assigns a clipping mask of the horse in the image, and a reflection of the input causes the same reflection in the output. }
	\label{fig:in_eq_eg}
\end{figure}

The information of this symmetry can be used to significantly improve models: in the case of finite semi-groups by effectively gaining $(|G| - 1) \times n$ extra data points; and for infinite groups more so by reducing high dimensional problems to low dimensional ones. Given some estimator $\hat{f}(x, \mathcal{D})$, we can include $G$-equivariance information by: (1) data augmentation: $\hat{f}_G( x, \mathcal{D}) = \hat{f}( x, G \cdot \mathcal{D})$ where $G \cdot \mathcal{D} = \{ (g\ \cdot X_i, g \star Y_i) : g \in G, (X_i, Y_i) \in \mathcal{D} \} $; (2) kernel symmetrisation: replacing a kernel $K$ with $K_G(x,y) = \frac{1}{|G|} \sum_{g \in G}K( g \cdot x, y)$; or by (3) feature averaging: $\hat{f}_G (x, \mathcal{D} ) = S_G \hat{f} = \tfrac{1}{|G|} \sum_{g \in G} g^{-1} \star \hat{f}(g \cdot x, \mathcal{D})$. In the case of infinite semigroups, the sums can be replaced with expectations for some $G$ valued random variable $g$.

In the case of feature averaging with a group $G$, if $f$ is $G$-invariant (so $f = S_Gf $ in $L^2(\mathcal{X})$) we can know immediately that
\begin{equation}
\label{eq:proj_bound}
\| S_G \hat{f}  - f \|_2^2 = \| S_G \hat{f} - S_G f \|_2^2 \leq \| S_G \|^2 \| \hat{f} - f \|_2^2 \leq \| \hat{f} - f \|_2^2  	
\end{equation}
i.e., the expected integrated squared error of the symmetrised estimator must be at most the error of the unsymmetrised model for any target $f$ because $S_G : L^2(\mathcal{X} ) \rightarrow L^2_G(\mathcal{X})$ is a projection to the $G$-invariant subspace $L_G^2(\mathcal{X})$ and thus has operator norm $1$. This and similar results have been shown by \cite{chen2020group} (for data augmentation), \cite{bietti2021sample} (for Kernel symmetrisation), and \cite{elesedy2021provably} and \cite{lyle2020benefits} (for feature averaging). 

 However, this is limited to cases where $f$ is known to be invariant or equivariant a priori. If we model using an incorrect symmetry, then the model will learn the wrong function: $S_G f \neq f$. In fact, the result of equation \ref{eq:proj_bound} depends on the true $G$-invariance of $f$. If $f$ is not $G$-invariant, then we know that:
\begin{equation}
	\| S_G \hat{f} - f \|_2^2 = \| S_G ( \hat{f} - f)  + S_G f - f \|_2^2 \overset{p}{\rightarrow} \|  S_G f - f \|_2^2 > 0
\end{equation}
so the feature averaged estimator loses the universal consistency of the original estimator by creating extra irreducible error.   This means that an inappropriately symmetrised model can have significantly worse performance, as shown in figure \ref{fig:non_inv_error}.

\begin{figure}[h]
	\centering
	\begin{subfigure}{.45\textwidth}
		\includegraphics[scale=0.29]{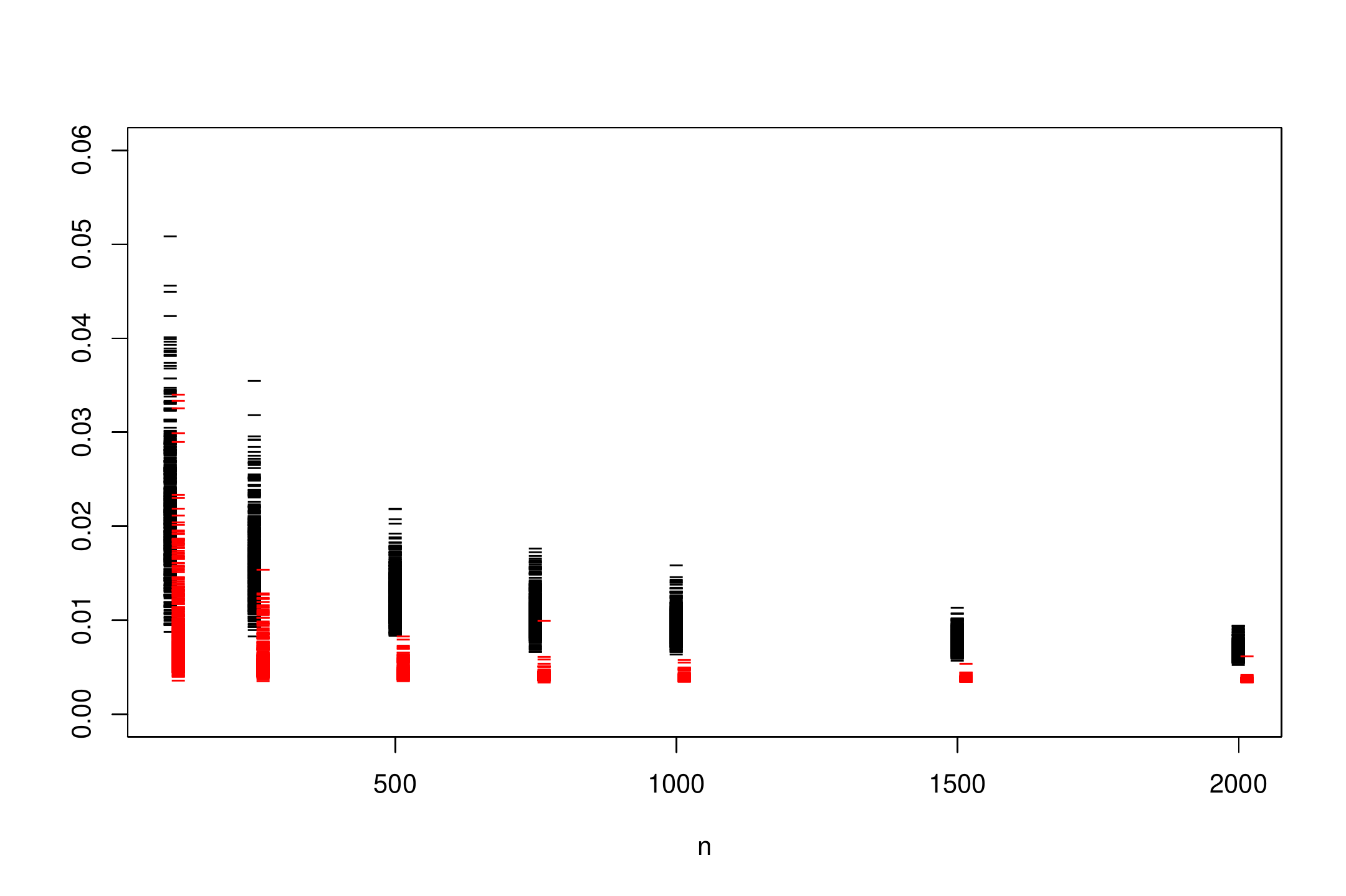}
		\caption{Test errors for a $G$ invariant target.} 
	\end{subfigure}
	\begin{subfigure}{.45\textwidth}
		\includegraphics[scale=0.29]{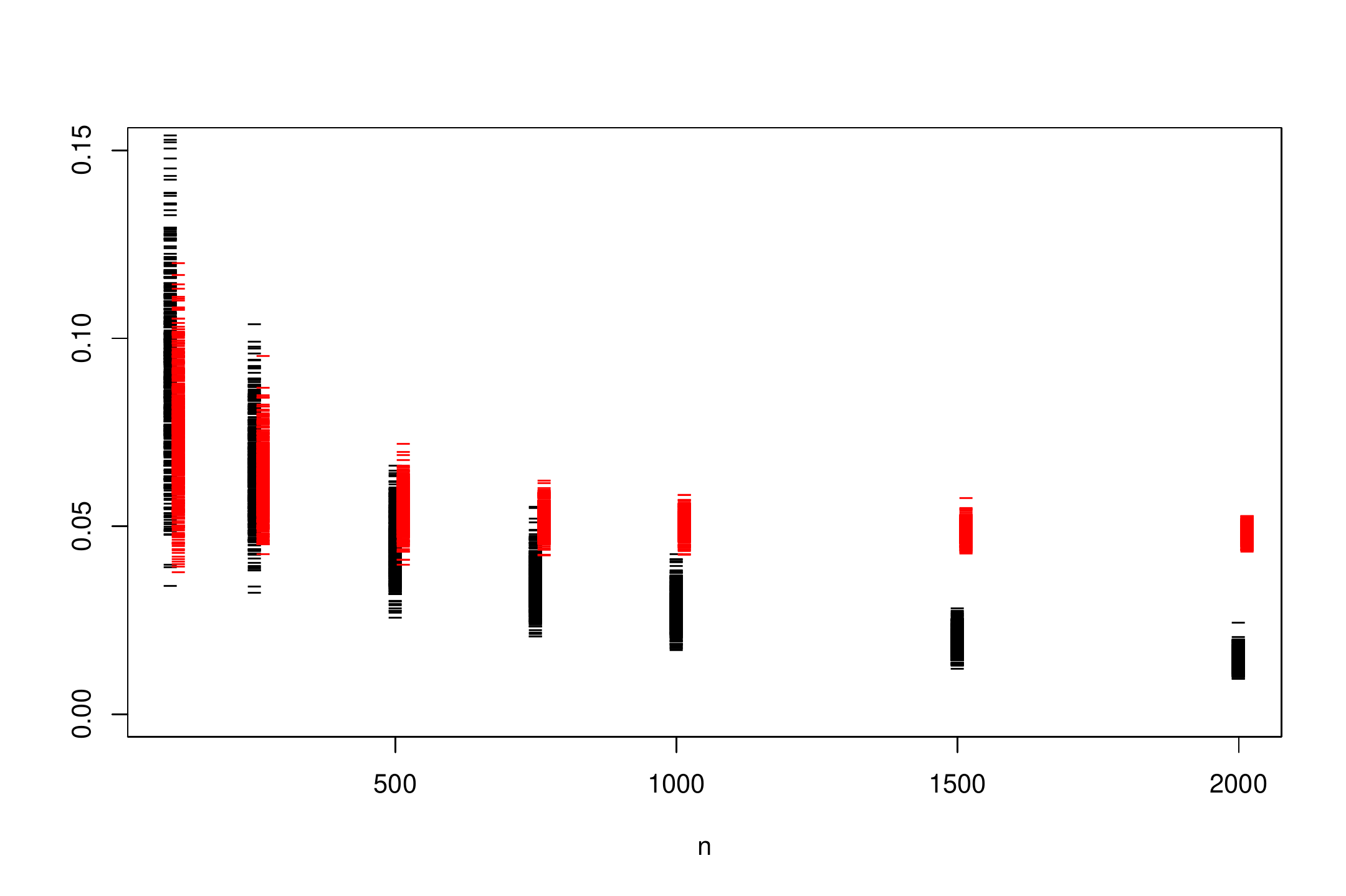}
		\caption{Test Errors for a non invariant target.} 
		\label{sfig:as_error}
	\end{subfigure}	
	\caption{Test errors of simulations a local constant estimator with an asymptotically optimal bandwidth and a rectangular kernel (black lines), and a kernel symmetrised modification (red lines). For the invariant regression function $f$ the symmetrised model always outperforms the original. For the non-invariant regression function, the symmetrised model has an irreducible error (fig \ref{sfig:as_error}). Full details of these simulations are available in appendix \ref{sapp:fig_2_code}.}
	\label{fig:non_inv_error}
\end{figure}

In this paper we present methods for testing for invariance and equivariance of a regression function $f$. This applies directly to classification setting as well via the probabilistic regression functions used in such models. This allows a modeller to quickly identify scenarios that symmetrised models are and are not applicable to, and allow for their more widespread use. 

These tests compare the dataset $\{ ( d_\mathcal{X}( X_i, X_j), | Y_i - Y_j | ) \}_{i,j = 1}^n$ to the symmetrised version $\{ ( d_{\mathcal{X}} (g \cdot X_i , X_j ), |g \star Y_i - Y_j | ) \}_{i,j = 1, g \in G}^n$. If $f$ is indeed $G$-equivariant, then these will be drawn from the same distribution, but if not we will see a higher spread in the symmetrised version. In particular, if we know that $f$ has a bound on it's variation, for example the Lipschitz $|f(x) - f(y)| \leq L \| x - y \|$, then we can quantify the probabilities of seeing the symmetrised points in the upper left region in figure \ref{fig:data_shift}. 

\begin{figure}[h]
	\centering
	\begin{subfigure}{.45\textwidth}
		\begin{tikzpicture}
		\node(a){\includegraphics[scale=0.42]{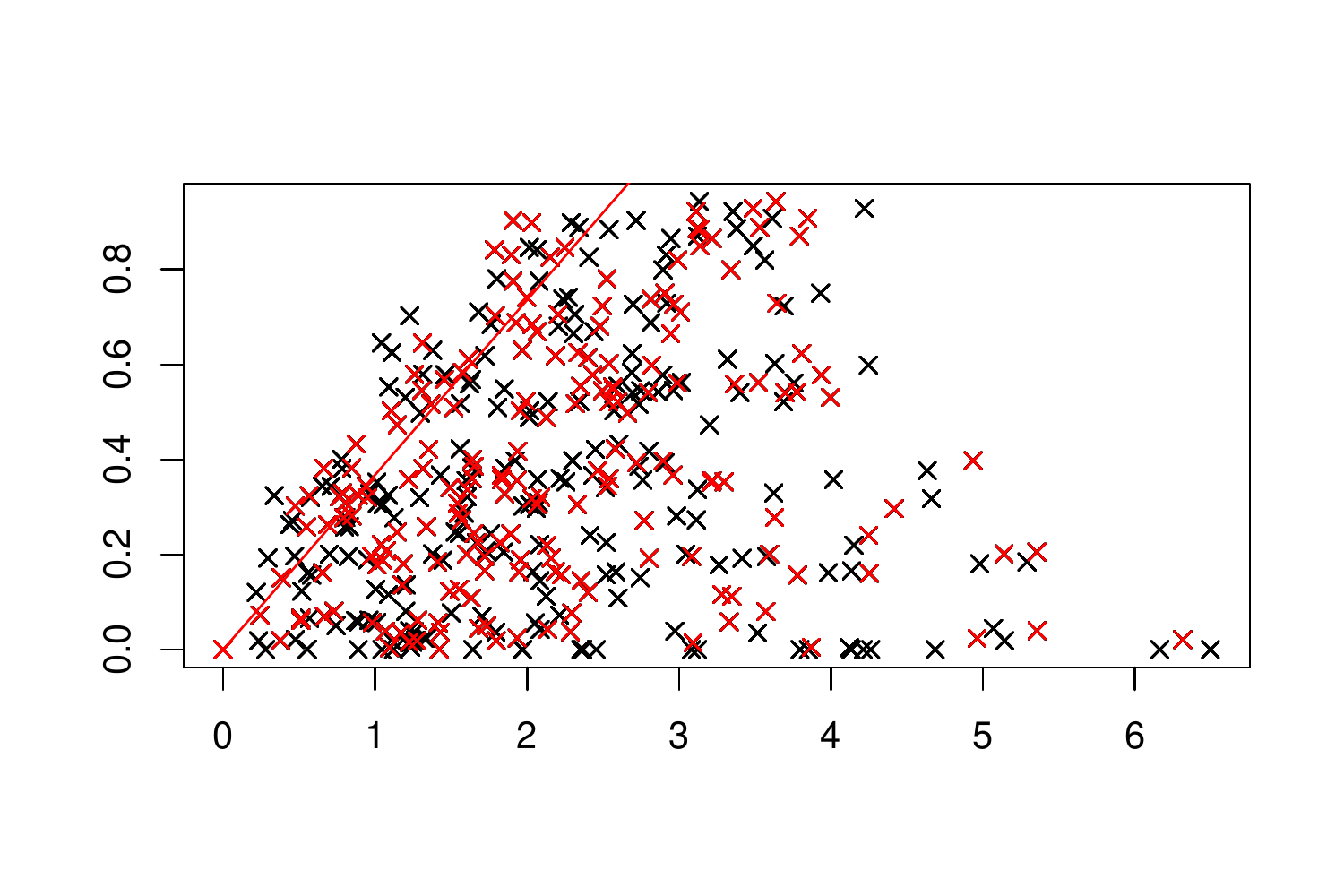}};
		\end{tikzpicture} 		\caption{Datapoints for $G \star \mathcal{D}$.} 
	\end{subfigure}
	\begin{subfigure}{.45\textwidth}
		\begin{tikzpicture}
		\node(a){\includegraphics[scale=0.42]{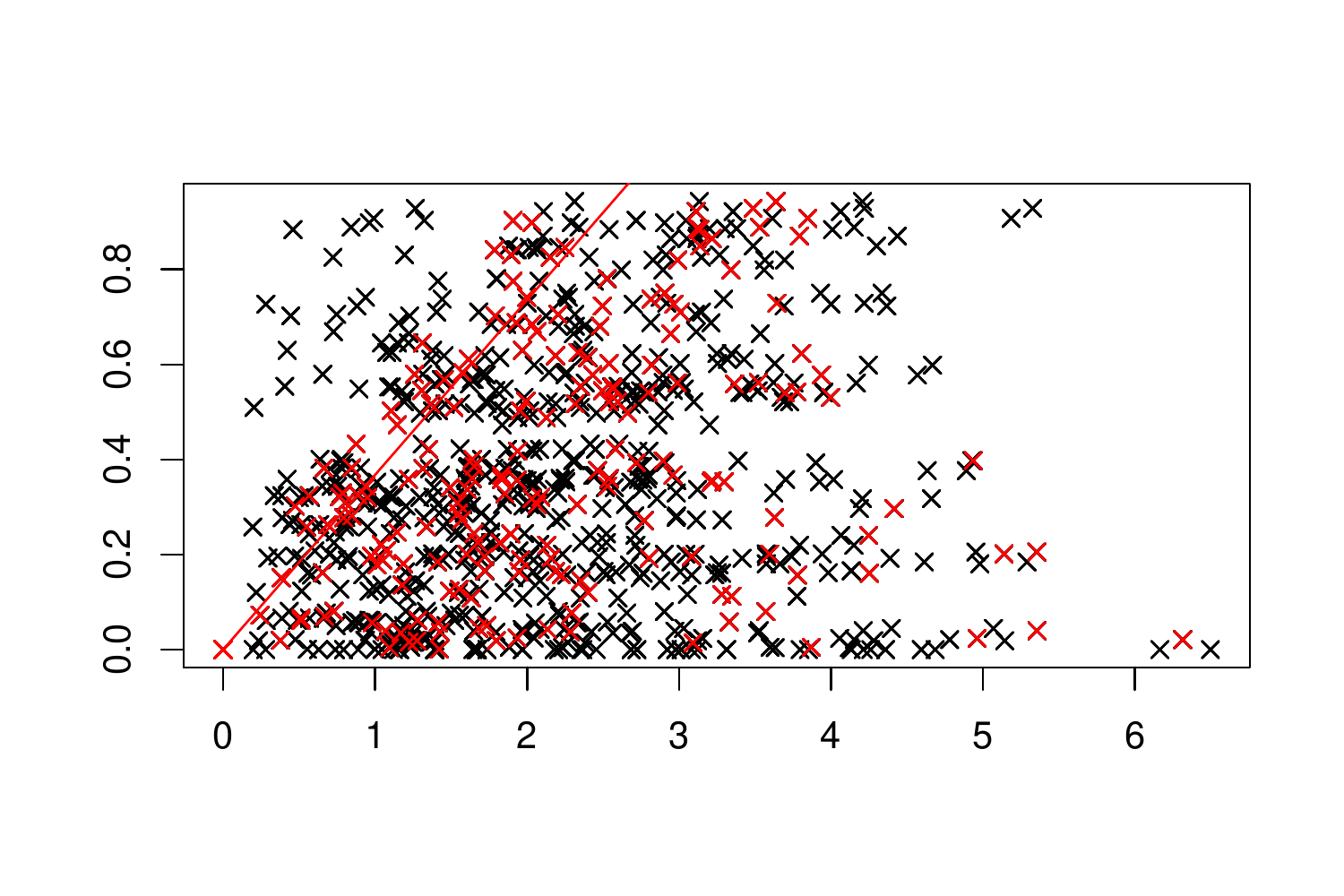}};
		\node at(a.center)[draw, red,line width=1pt,ellipse, minimum width=50pt, minimum height=30pt,rotate=45,yshift=45pt,xshift=-20]{};
		\end{tikzpicture} 
		\caption{Datapoints for $G \cdot \mathcal{D}$. } 
		\label{sfig:as_error}
	\end{subfigure}	
	\caption{Output distances $|Y_i - Y_j|$ vs input distances $d_\mathcal{X}(X_i, X_j)$, as well as a red line for the variation bound $V(x,y)$, for datasets from section \ref{ssec:sims}. Points from $\mathcal{D}$ only are shown in {\cred red}. When $f$ is not $G$-invariant, as in figure \ref{sfig:as_error}, we see many more points that are further above the red line. These points are circled in subfigure \ref{sfig:as_error}.}
	\label{fig:data_shift}
\end{figure}

This behaviour of the pairwise distances allows us to do the following: 
\begin{itemize}
	\item We introduce (in Section \ref{sec:testing}) the \textit{asymmetric variation test} $(G, \cdot, \star)$-equivariance of $f$, when $\star$ is a linear action and $\mu_\epsilon$ (i.e., the distribution of the i.i.d. noise variables $\epsilon_i$) is $(G, \star)$-invariant. This test relies the practitioner having a known bound on the variation of the regression function $f$. The key idea is that the asymmetry often breaks this bound on the variation, so we will more often see positive values of $D_{ij}^g = | g \star Y_i - Y_j| - V(g \cdot X_i, X_j)$ where $V(u,w) = \max_{f \in \mathcal{F}} | f(u) - f(w) |$ over the class of bounded variation functions $\mathcal{F}$, which the practitioner assumes that $f$ belongs to.  We show that this is a consistent hypothesis test under mild conditions on the noise in section \ref{ssec:consistency}.
	\item We introduce (in Section \ref{sec:assumptionless}) the \textit{permutation variant} of the asymmetric variation test. This drops the requirement of knowledge of $V$ by using a permutation style test on the quantities $S_{ij}^g = |g \star Y_i - Y_j | / \mathcal{V}(g \cdot X_i, X_j)$. Here $\mathcal{V}(x,y)$ is a function for which $|f(x) - f(y)| = O(\mathcal{V}(x,y))$ for all $f$ and $x,y \in \mathcal{X}$. This is essentially assuming only the order of $V$ and allows us to ignore $V$'s scale, but requires much larger computational cost. 
\end{itemize}

\section{Asymmetric Variation Test}
\label{sec:testing}

Let $G$ be a group that acts on both $\mathcal{X}$ and $\mathcal{Y}$, with the action on $\mathcal{Y}$ linear and for which $g \star \epsilon_i \overset{D}{=} \epsilon_j$ for all $g \in G$ and $i,j$. An important example of such an action are the axis permutations, including the image transformations of figure \ref{fig:in_eq_eg} (when the noise is isotropic). It also includes all geometric invariances, as the trivial action is linear and $\epsilon$ preserving. 

We first aim to test the hypothesis $H_0 : f$ is $(G, \cdot, \star)$-equivariant against $H_1 : f$ is not $(G, \cdot, \star)$-equivariant. The methodology outlined relies on two key assumptions: (1) a known bound on local variations of $f$; and (2) a known concentration inequality for the noise variables. There are two key choices for the practitioner: (1) a $G$ valued random variable $g$; and (2) a threshold value $t$ for the concentration inequality.

\subsection{Methodology}

Suppose that $\mathcal{F}$ is a class of bounded variation for which the practitioner assumes $f \in \mathcal{F}$, and let $V(x,y) = \sup_{f \in \mathcal{F} }  |f(x) - f(y) |$. An example of such a class are $\alpha$-H\"{o}lder continuous functions $\mathcal{F}(L, \alpha)$ with $|f(x) - f(y) | \leq L d_\mathcal{X}(x,y)^\alpha$ for $\alpha \in (0,1]$, for which $V(x,y) = L d_\mathcal{X}(x,y)^\alpha$. 

Let the independent mean zero additive noise $\epsilon_i = Y_i - f(X_i)$ be such that $ \PP( | \epsilon_i - \epsilon_j | > t ) \leq p_t$. An example of this would be $p_t = \frac{ 2 \sigma }{ t  } \tfrac{ \exp( - t^2 / 4 \sigma^2 )}{\sqrt{2\pi}}$ for iid gaussian $\epsilon_i$ with variance $\sigma^2$ (Proposition 2.1 of \cite{adamsMA3K0notes}) in each dimension of $\mathcal{Y}$. Let $g \sim \mu_g$ be any $G$ valued random variable. 

Consider bounding the $p$-value $\PP( \mathcal{D} \text{ or less likely} \mid H_0 )$ via algorithm \ref{algo:known}. If $f$ is $(G, \cdot, \star)$-equivariant then  
\begin{equation}
\label{eq:like_bound}
|g \star Y_i - Y_j| \leq |f(g \cdot X_i) - f(X_j)| + |g \star \epsilon_i - \epsilon_j| \leq V(g \cdot X_i, X_j) + |g \star \epsilon_i - \epsilon_j|,
\end{equation}
we know that $D_{ij} \leq |g \star \epsilon_i - \epsilon_j| \overset{D}{=} | \epsilon_i - \epsilon_j |$. For fixed $m$, the $D_{I(j) J(j)}^{g_j}$ are asymptotically independent (because it is vanishingly unlikely that we sample the same $g_j \cdot X_{I(j)}$ or that two $g_j \cdot X_{I(j)}$ share a nearest neighbour). Thus under the null hypothesis, $N_t^g$ is stochastically bounded by a $\mathrm{Binom}(m, p_t)$ variable, which allows us to bound the $p$-value from above by the return value $p_{val}$.

\begin{algorithm}[h]
\caption{Asymmetric Variation Test}
\label{algo:known}
\begin{algorithmic}[1]
\Procedure{AsymVarTest}{$\mathcal{D}$, $V$, $\mu_g$, $t$, $p_t$, $m$}
	\For{$j \in \{1, \dots, m \}$}
		\State $g_j \leftarrow \mathrm{Sample}( \mu_g )$
		\State $I(j) \leftarrow \mathrm{Sample}( \{1, \dots, n \} )$
		\State $J(j) \leftarrow$ Index of Nearest Neighbour to  $g \cdot X_{I(j)}$ in $\{ X_j \}_{j = 1}^n$
		\State $D_{I(j) J(j)}^{g_j} \leftarrow | g \star Y_{I(j)} - Y_{J(j)} | - V( g \cdot X_{I(j)}, X_{J(j)} )$
	\EndFor
	\State $N_t^g \leftarrow | \{ D_{I(j) J(j)}^{g_j} \geq t \} |$
	\State $p_{val} \leftarrow \sum_{k = N_t^g}^{m} \binom{m}{k} p_t^k (1 - p_t)^{m - k}$
	\State \Return $p_{val}$
\EndProcedure
\end{algorithmic}	
\end{algorithm}

\subsection{Choices of $t$ and of $g$}

The methodology presented here works for any choice of $t$ and variable $g$, though particular choices of these will affect the power of the test. For example, if we choose $t$ such that $p_t \geq 1$ then our $p$-value will always be $1$. Similarly if $g = e \in G$ almost surely then we will also only reject with probability at most $\alpha$. 

For the choice of $t$, we suggest calculating $N_t^g$ from the sample of $D_{I(j)J(j)}^{g_j}$ at some grid of $t$ values $t_0 < t_1 < \cdots < t_k$ with the values of $p_{t_i}$ spread over the interval $(0,1)$, and then taking the $p$-value of $H_0$ as the minimum $p$-value of each $N^g_{t_i}$. This is justified as the information of the test is entirely contained in the set $\{ D_{I(j)J(j)}^{g_j} \}_{j = 1}^m$, i.e., since the $p\text{-value}$ is at most $\PP( N_t^g \geq k_t \mid H_0 )$ 
for all $t$, we can take an infimum over $t$ on the right hand side. 

For the choice of $g$, we suggest using a uniform distribution only on some set of generators of $G$. This means that we don't sample the identity or other elements that generate subgroups of $G$ that $f$ may be equivariant under, but if there is an element that breaks the equivariance then one of the generators will too.

\subsection{Consistency of this test}
\label{ssec:consistency}

Under some mild conditions on the noise distribution, and with $m$ set at $n$ for all $n$, we can prove that the asymmetric variation test is consistent. 

\begin{prop}
	\label{prop:cons}
	Set $m = n$ and fix $t > 0$. Suppose that the law of $X$ has a dense support on $\mathcal{X}$. Suppose that the noise concentration bound is tight: $p_t = \PP( |\epsilon_i + \epsilon_j| > t)$ and that $\epsilon$ admits a density $f_Y$ with respect to Lebesgue measure on $\mathcal{Y}$ that is decreasing in $|y|$. Then the asymmetric variation test is consistent, i.e. $p_{val} \mid H_1 \overset{p}{\rightarrow} 0$.
\end{prop}
	
The proof can be found in Appendix \ref{app:cons_proof}. The condition on the support of $\mu_X$ amounts to restricting $\mathcal{X}$ to the closure of the support as a practitioner would usually do. The condition that the noise admits a density is satisfied in many usual cases in regression (e.g. Gaussian noise). The condition that we can bound the concentration of the noise tightly is somewhat restrictive, but reflects the difficult of the problem - if the asymmetry is obscured by more noise then it is much more difficult to identify it. The condition could be relaxed to account for the noise threshold that still allow of consistency, but since in practice we cannot calculate this threshold without knowledge of $f$ we omit it from this paper.

\section{Permutation Variant of the Asymmetric Variation Test}
\label{sec:assumptionless}

The asymmetric variation test relies on both the existence of, and the knowledge of, the bound $V(x,y)$. In this section we show that we can remove some of the requirement of the knowledge at the cost of computational power. 

This test still has the assumption of a known order for the variation bound, but not not require any knowledge of the noise variables $\epsilon_i$ (other than that they are iid). The practitioner still has a choice of the random variable $g$, but now chooses a quantile $q$ instead of threshold(s) $t$.

Suppose that we know only the order of the bound $V$, i.e., we know some $\mathcal{V}(x,y)$ such that for all $f \in \mathcal{F}$ there exists some (unknown) $L_f$ with $|f(x) - f(y)| \leq L_f \mathcal{V}(x,y)$. Clearly any known $V$ satisfies this property, but it is weaker in that we do not need to know that bound exactly. A key example of $\mathcal{V}$ is $d_\mathcal{X}(x, y)^\alpha$ for $\alpha$-H\"{o}lder continuous functions in any $\mathcal{F}(L, \alpha)$ (whereas we would need a constant multiple of this for a class of particular $\alpha$-H\"{o}lder continuous functions). 

Let $S_{ij}^g = |g \star Y_i - Y_j| / \mathcal{V}(g \cdot X_i,X_j)$. Consider collecting $S^k = \{ S_{I(j)J(j)}^g \}_{j = 1}^m$ as in algorithm \ref{algo:known}, but where $J(j) \sim U( \{1, \dots, n\}$ instead of being the index of the nearest neighbour, and let $A_k$ be the $q^{th}$-quartile of this set for some chosen $q \in (0,1]$. Under the null hypothesis $H_0: f$ is $(G, \cdot, \star)$-equivariant, the distributions of $S_{ij}^g$ will be the same as $S_{ij}^e$, so we can run a permutation test (See \cite{good2005permutation}) on the set $\{A_k \}_{k = 1}^B$, comparing them to the $q^{th}$-quantile $A_0$ of $S^0 = \{ S_{I(j)J(j)}^e \}_{j = 1}^m$ (sampled in the same way, but with $g = e$ almost surely). We can approximate the $p$-value by the proportion of $A_k \leq A_0$. This is described in algorithm \ref{algo:perm}. 

\begin{algorithm}[h]
\caption{Permutation Variant of Asymmetric Variation Test}
\label{algo:perm}
\begin{algorithmic}[1]
\Procedure{PermVarTest}{$\mathcal{D}$, $\mu_g$, $\mathcal{V}$, $q$, $m$, $B$}
	\For{$k \in \{1, \dots, B \}$}
		\For{$j \in \{1, \dots, m \}$}
			\State $g_j \leftarrow \mathrm{Sample}( \mu_g )$
			\State $I(j), J(j) \leftarrow \mathrm{Sample}( \{1, \dots, n \} )$
			\State $S_{I(j) J(j)}^{g_j} \leftarrow | g_j \star Y_{I(j)} - Y_{J(j)} | / \mathcal{V}( g_j \cdot X_{I(j)}, X_{J(j)} )$
		\EndFor
		\State $A(k) \leftarrow \mathrm{quantile}( \{ S_{I(j) J(j)}^{g_j} \}_{j = 1}^m , q )$ \Comment{R syntax}
	\EndFor
	\For{$j \in \{1, \dots, m \}$}
		\State $I(j) \leftarrow \mathrm{Sample}( \{1, \dots, n \} )$
		\State $J(j) \leftarrow$ Index of Nearest Neighbour to  $g \cdot X_{I(j)}$
		\State $S_{I(j) J(j)}^{e} \leftarrow | Y_{I(j)} - Y_{J(j)} | / \mathcal{V}( X_{I(j)}, X_{J(j)} )$
	\EndFor
	\State $A_0 \leftarrow \mathrm{quantile}( \{ S_{I(j) J(j)}^{e} \}_{j = 1}^m , q )$
	\State $p_{val} \leftarrow |\{ k : A(k) \leq A_0 \}| / B$
	\State \Return $p_{val}$
\EndProcedure
\end{algorithmic}	
\end{algorithm}

\subsection{Finite sample effects and choice of quantile $q$}

Whilst it is true that $S^g_{ij} \overset{D}{=} S^e_{ij}$ for independent $X_i,X_j$, the finite sample estimates of these quantities are not equal. In fact, the distribution of $S^g_{ij}$ can be biased upwards because the action on $\mathcal{X}$ allows us to see more points with smaller $d_\mathcal{X}(X_i,X_j)$. As $m$ increases, we are more and more likely to see the outlying values of $S_{ij}^g$ compared to $S_{ij}^e$. This can cause a bias towards rejection in this permutation test, i.e., it is liberal for small values of $n$. This inexactness is a known issue with permutation tests, for example when testing the variance of univariate samples (as in \cite{good2005permutation}, \S 3.7.2), and approximate permutation tests as used here are still shown to be successful and useful. 

These problems are alleviated somewhat by using the quantile $q \in (0,1]$, as this is less sensitive to the outlying values than picking $q = 1$ (i.e., just going with the maximum of the $S_{I(j)J(j)}^{g_j}$). This does mildy reduce the power of the test, but improves the specificity significantly. We have found in simulations that using $q = 0.95$ works well in practice (see figure \ref{fig:effect_of_q} in appendix \ref{sapp:effect_q}).

\section{Numerical Experiments}

Here we run simulations of both the asymmetric variation test and the permutation variant for low dimensional examples and for the orientation of digits in the MNIST dataset ($d = 784$). All code is available in the supplementary material and on GitHub at \url{https://github.com/lchristie/testing_for_equivariance}.

\subsection{Simulations in low dimensions}
\label{ssec:sims}
Let $d \geq 2$ and set $\mathcal{X} = \RR^d$. Take $X_i \iid N(0,4I_d)$ and $\epsilon_i \iid N(0,\sigma^2 I_d)$. Consider the functions $f_d : \RR^d \rightarrow \RR$ be given by $(x_1, x_2, \dots, x_d) \mapsto \exp( - |x_1| )$. Let $G = \langle R_{\pi/2} \rangle$ act on $\RR^d$ via the distinct actions generated by $ R_{\pi/2} \cdot x = ( - x_2, x_1, x_3, \dots, x_d )$, $R_{\pi/2} \star x = (-x_1, -x_2, x_3, \dots, x_d)$, and let it act on $\RR$ by the trivial action $R_{\pi / 2} \bullet y = y$. Then we know that $f_d$ is $(G, \star, \bullet)$-invariant but not $(G, \cdot, \bullet)$-invariant for all $d$. 

We simulated tests of each of the hypothesis $H_0^{(1)} : f_d$ is $(G, \star, \bullet)$-invariant and $H_0^{(2)} : f_d$ is $(G, \cdot, \bullet)$-invariant, using both tests (algorithms \ref{algo:known} and \ref{algo:perm}). The estimated power graphs and empirical sizes are plotted in figure \ref{fig:est_power}, containing rejection probabilities at significance level $\alpha = 0.05$. We ran 100 simulations for each combination of $n = m \in \{ 20,30,40,50,60,70,80,90,100,120,150,200,250,300 \}$. The asymmetric variation tests had $t = 2 \sigma = 0.1$, $p_t = \frac{ 2 \sigma }{ t  } \tfrac{ \exp( - t^2 / 4 \sigma^2 )}{\sqrt{2\pi}}$ and the permutation variants had $B = 100$ and $q = 0.95$. We tested with $g_j \iid U( R_{\pi/2},  R_{\pi/2}^2,  R_{\pi/2}^3 )$ for all tests. Further simulations for other combinations of $n \neq m$ are available in appendix \ref{app:sims}, as well as for other regression functions.  
\begin{figure}[h]
	\centering
	\begin{subfigure}{.45\textwidth}
		\includegraphics[scale=0.42]{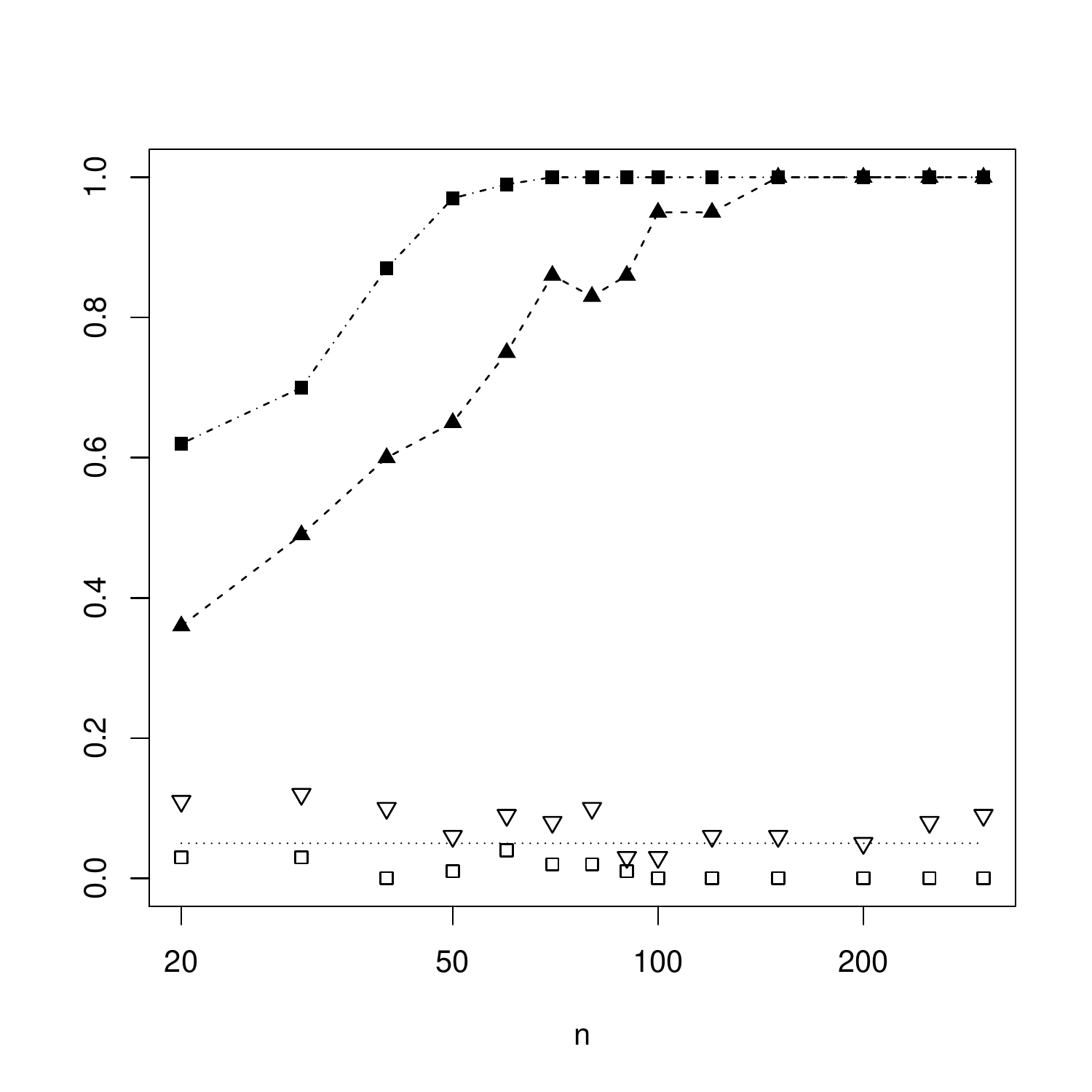}
		\caption{Rejection Proportions for $f_2$.} 
	\end{subfigure}
	\begin{subfigure}{.45\textwidth}
		\includegraphics[scale=0.42]{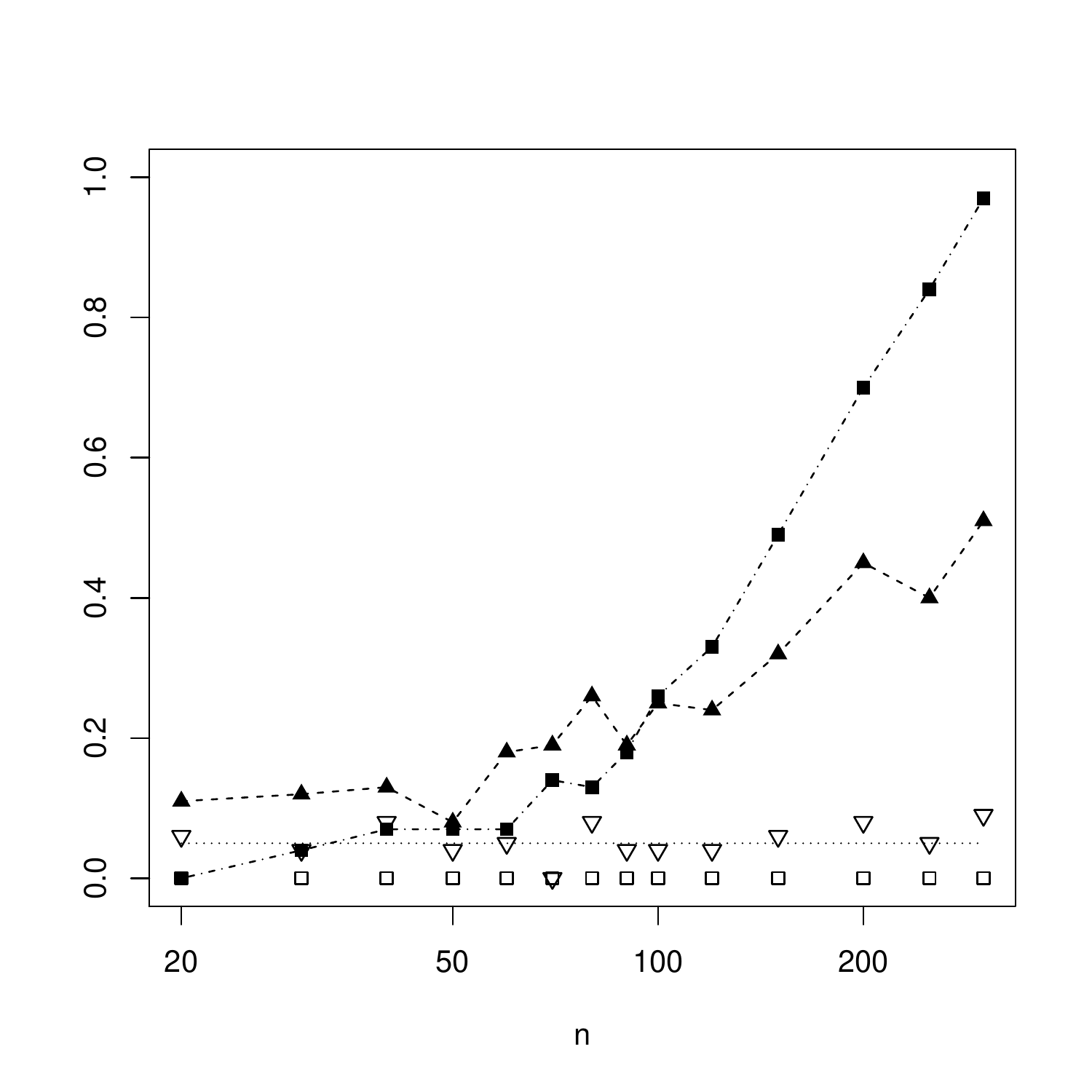}
		\caption{Rejection Proportions for $f_{4}$.} 
		\label{sfig:rp_f_4}
	\end{subfigure}	
	\caption{Rejection proportions for $f_d$. Squares are for the asymmetric variation test, and triangles for the permutation variant. Filled shapes are the estimated power of $H_0^{(2)}$ and unfilled are the empirical sizes. The dashed horizontal line is the significance level $\alpha$.}
	\label{fig:est_power}
\end{figure}
We see that both tests converge to an estimated power $\hat{\PP} ( Reject \mid H_1 ) \rightarrow 1$, with the asymmetric variation test more powerful. Both tests have roughly (up to the variation from 100 simulation) the correct empirical size, though the permutation variation does appear to be slightly liberal for $n \leq 50$.

\subsection{MNIST digit orientation test}

Consider the symmetries of images of the digits 3 and 8. It is clear that rotations of both preserve the classification, however one can argue that only the 8 is invariant to horizontal reflections - a reflected 3 would be an $\mathcal{E}$ (in the same way that a reflected p is a q), and so can be considered oriented in a way that the 8 is not. We call the in digits $\mathcal{O} = \{ 2, 3, 4, 5, 6, 7, 9 \}$ the \textbf{oriented digits} and the others $\{ 0, 1, 8 \}$ the \textbf{non-oriented digits}. Here we test for the symmetry of the probabilistic labels regression function using the asymmetric variation test. 

We subset the MNIST dataset \cite{lecun1998gradient} (Available under a CC BY-SA 3.0 licence) into particular characters, called $\mathcal{D}_{X}(n)$ for $n \in \{0, 1, \dots, 9 \}$. We then split these in half uniformly at random to form sets $\mathcal{D}_{X}^1(n)$ and $\mathcal{D}^R_{X}(n)$. We then apply $b$ (the reflection through the vertical line) from $D_4 = \langle a, b : a^4 = b^2 = 1, ab = ba^{-1} \rangle$ to the elements of $\mathcal{D}_X^R(n)$ and then assign labels $Y_i = 1$ for every for $X_i \in \mathcal{D}_X^{1}(n)$ for all $n$ and $\mathcal{D}_X^R(n)$ if $n$ is non-oriented, and the label $Y_i = 0$ for every other $X_i$. This gives datasets
\begin{equation}
\mathcal{D}(n) = \{ (X_i, 1) : X_i \in \mathcal{D}_X^{1}(n) \} \cup \{ (X_i, \mathbf{1}_{n \not\in \mathcal{O} }) : X_i \in \mathcal{D}_X^{R}(n) \}
\end{equation}
Here the labels $Y_i$ are the probabilities that $X_i$ will be recognised as the digit $n$, and so we are learning the function $f_n : [0,1]^{784} \rightarrow [0,1]$ that assigns such probabilities. We assume that there is no noise, so $p_t = 0$ for all $t$. Lastly we need to pick $V(x,y)$, which we estimate with the reciprocal of the minimal distances between digits in and out of each class. 

We then test each of the datasets $\mathcal{D}(n)$ for symmetry of $D_4$, and for the subgroups $\langle a \rangle$ of rotations and $\langle b \rangle$ of horizontal reflection. These are done with $g^1 \sim U( \{ a, b \} )$, $g^2 \overset{D}{=} a$ and $g^3 \overset{D}{=} b$. The number out of $m = 1000$ samples that have $D_{I(j)J(j)}^g > 0$ is reported in table \ref{tab:mnist}. Since $p_t = 0$, $N_0 > 0$ if and only if the $p$-value is 0. All tests were successful - for all oriented digits all $D_4$ and $\langle b \rangle$ symmetries were rejected whilst all $\langle a \rangle$ symmetries were accepted. Non-oriented digits can never cause a rejection because the responses are all equal. The larger numbers reflect the ease of identifying the asymmetry, at least partly because the estimated bound $V$ is sharper for some of the classes.

\begin{table}[h]
\centering
\caption{Calculated $N_0$ for MNIST digit orientation tests.}
\label{tab:mnist}
\onehalfspacing
\begin{tabular}{c|ccccccc}
		 Test for:    		  	 & 2  & 3  & 4  & 5  & 6   & 7  & 9 \\ \hline \hline
		 $D_4$ 			  		 & 6  & 20 & 16 & 71 & 164 & 54 & 9  \\
		 $\langle a \rangle$ 	 & 0  & 0  & 0  & 0  & 0   & 0  & 0  \\
		 $\langle b \rangle$ 	 & 2  & 9  & 7  & 23 & 49  & 27 & 5  \\
\end{tabular}
\end{table}

\section{Conclusion}

\textbf{Comparison of test powers.} The permutation variant has a key advantage over the standard Asymmetric Variation Test: if the bound $V$ is weak for the true regression function $f$ (e.g., if $V(x,y) = d_\mathcal{X}(x,y)$ but $f$ is $0.5$-Lipschitz) then the power of the asymmetric variation test is artificially reduced. Here we do not make such an assumption, and we effectively estimating the minimal constant $C$ such that $|f(x) - f(y)| \leq C V(x,y)$ with $A_0 \times \tfrac{V(x,y)}{\mathcal{V}(x,y)}$. This means that the computational power used is not just traded for the assumption, we also gain some power for regression functions not on the boundary of $\mathcal{F}$. However, it suffers from the loss of the nearest neighbour sampling; these points are much more likely to tell us about the asymmetry which boosts the power of the standard test, as seen in figure \ref{sfig:rp_f_4}.

\textbf{Testing in the other direction.} The test as presented mainly gives confidence to the rejections - i.e., it is only useful to show that a given regression function is not equivariant. One could hope to test with the hypotheses reversed, where we could then give statistical confidence to the presence of a symmetry. Unfortunately this is problimatic because the class of non-symmetric functions is dense in $L^2(\mathcal{X})$, which makes it impossible to distinguish between a symmetric and non-symmetric function in finite samples. This could be relaxed to test for a hypothesis of the form $H_0 : \| f - S_G f \| \geq \delta$ for some threshold $\delta$, and this would be an interesting question for future work. 

\textbf{Computational Complexity.} These tests are remarkably fast, requiring $O(nm)$ (for the asymmetric variation test) or $O(m^2B)$ (for the permutation variant) computations only. This depends on the dimensions of $\mathcal{X}$ and $\mathcal{Y}$ only through the evaluations of $d_\mathcal{X}$ and $| \cdot |$. These are usually $O(d)$ (e.g., Euclidean norms, Minkowskii distances.). The asymmetric variation test evaluated on MNIST with $n = 6131$, $m = 1000$, and $d = 784$ in only 75 seconds using a 2020 MacBook pro with a 2GhZ quad core i5 with 16GB of RAM. 

\textbf{Limitations.} The main limitations of this test are the assumptions that the data is i.i.d., and that we have knowledge of $\mathcal{V}$ or $V$. The first is minor, and most of the results can be recovered as long as the noise is independent and each $\epsilon_i$ satisfies the $p_t$ bound. The second requires some care, but can be tested on the data as we did with MNIST. This tests ``fails safely'', in the sense that a false positive (i.e. $p_{val} \mid H_0 \leq \alpha$) is not a problem for the practitioner - they can result to standard non-equivariant tools. A false negative (i.e. $p_{val} \mid H_1 \geq \alpha$) may waste time building a model, but without this test that time would be guaranteed to be wasted. 

\textbf{Future work.} There are two ways in which these tests could be extended. The first is for non-linear actions, and the second is to consider the effect of a non $G$-invariant noise distribution. Both of these extensions will require work to rebound the distribution of $|g \star \epsilon_i - \epsilon_j |$, as the bound \ref{eq:like_bound} relies on these assumptions. One could also look to develop a further test that directly estimates the bound $V$, rather than using the permutation approach that does this implicitly.

\newpage
\section*{Broader Impact}

This testing framework is the first way to quantify the presence of invariance and equivariance in datasets. This work will benefit anyone that wishes to include symmetric information in complex machine learning models. It allows people to quickly test for the existence of a symmetry, which should prevent these models being used in cases where they shouldn't. This will give some confidence to their use as well, and provides some direct support when we do see lower test errors. This directly alleviates some of the concern raised in the "Broader impact statement" of \cite{chen2020group}. 

This test does require some care, as with all hypothesis tests. Firstly the practitioner needs to be careful about placing assumptions on $V$, $\mathcal{V}$, and $\epsilon$ (in particular). They need to understand that a high $p$-value does not constitute strict evidence for the symmetry, and evidence from domain specific knowledge is still required to understand whether the symmetry is truely present.

\begin{ack}

This work was supported by the University of Cambridge Harding Distinguished Postgraduate Scholars Programme. This work was also supported by Engineering and Physical Research Council grant EP/T017961. We thank the StatsLab and the rest of the Department of Pure Mathematics and Mathematical Statistics at the University of Cambridge, and the Cantab Capital Institute for the Mathematics of Information. 

\end{ack}

\bibliography{bibtex.bib}{}
\bibliographystyle{acm}

\newpage
\appendix

\section{Appendix: Proof of Consistency for the Asymmetric Variation Test}
\label{app:cons_proof}

This appendix covers the proof of Proposition \ref{prop:cons} in section \ref{ssec:consistency}. We first prove a number of supporting lemmas.

\begin{Lem}
If the support of $\mu_X$ is dense in $\mathcal{X}$ then $d_\mathcal{X} ( g \cdot X_{I(1)}, X_{J(1)} ) \rightarrow 0$ in probability. 
\end{Lem}

\begin{proof}
	The probability that $d_\mathcal{X} ( g \cdot X_{I(1)}, X_{J(1)} ) \geq \eta$ given $I, g$ is equivalent to a binom variable with $n$ trials and probability of success $p = \mu_X( \mathcal{X} \setminus B (g \cdot X_{I(1)}, \eta)$ being $0$. Since the support of $\mu_X$ is dense in $\mathcal{X}$, $\mu_X( B (g \cdot X_{I(1)}, \eta) ) > 0$, and so   . As $n \rightarrow \infty$, this clearly goes to $0$ for all $\eta$, and for all $I, g$. 
\end{proof}

\begin{Lem}
	\label{lem:sym_props}
	If $X$ and $Y$ are independent real random variables, and $Y$ is symmetrically distributed around $0$ (so $F_Y(t) = 1 - F_Y(-t)$ for all $t$) and admits a density $f_Y$ with respect to Lebesgue measure that is decreasing in $|y|$. Then $\PP( |X + Y| \geq t) > \PP( |Y| \geq t )$. 
\end{Lem}

\begin{proof}
	Let $F_Y$ be the distribution function of $Y$. Then $\PP( X + Y \geq t) = \EE_X (\PP( Y \geq t - X \mid X) ) = 1 - \EE_X( F_Y(t- X ) )$. Thus 
	\begin{align}
		\PP( |X + Y |\geq t ) &= \PP( X + Y \geq t) + \PP( X + Y \leq -t ) \\
			&= 2 - \EE_X( F_Y(t - X) + F_Y(t + X) )
	\end{align}
	Let $\phi_t(x) = F_Y(t - x) + F_Y(t + x)$. Since $Y$ is absolutely continuous with density function $f_Y$, we can see that $\phi_t$ has a critical point at $x = 0$, moreover,	if $x > 0$ then
	\begin{equation}
		\phi'_t(x) = f_Y(t + x) - f_Y(t - x) = f_Y(t + x ) - f_Y(-t + x ) <  f_Y(t + x ) - f_Y(-t - x) = 0  
	\end{equation}
	and similarly $\phi'_t(x) > 0$ if $x < 0$. Thus $x = 0$ is a maximum and so we have
	\begin{equation}
		\PP( |X + Y |\geq t ) = 2 - \EE_X( F_Y(t - X) + F_Y(t + X) ) \geq 2 - \EE_X( F_Y(t) + F_Y(t) ) = \PP( |Y| \geq t ).
	\end{equation}
	as required.
\end{proof}

\begin{Lem}
	\label{lem:binom_likelihood}
	If 	$X \sim Binom(n, p)$ and $Y \sim Binom(n, q)$ with $p < q$ , then $F_X( Y) \overset{p}{\rightarrow} 1$
\end{Lem}

\begin{proof}
	First note that $F_X(Y) = F_{ \tilde{X} } ( \tfrac{Y - n p}{ \sqrt{n p (1 - p) } } )$ where $\tilde{X} = \tfrac{X - n p}{ \sqrt{n p (1 - p) } }$, which has $\tilde{X} \overset{D}{\rightarrow} N(0, 1)$ by the de Moivere - Laplace theorem (DLT). But also
	\begin{equation}
		\tilde{Y} = \frac{Y - n p}{ \sqrt{n p (1 - p) } } = \left( \frac{Y - n p}{ \sqrt{n q (1 - q) } } + \sqrt{n} \frac{q -  p}{ \sqrt{ q (1 - q) } } \right) \frac{\sqrt{q (1 - q)} }{ \sqrt{ p (1 - p) } }
	\end{equation} 
	Again by DLT, the first term converges to $N(0, \tfrac{q (1 - q) }{p (1 - p} )$, but the second term diverges to $+\infty$ (as $q > p$). Thus $\tilde{Y} \overset{p}{\rightarrow} \infty$ and so clearly $F_{\tilde{X} }( \tilde{Y} ) \rightarrow 1$, as required. 
\end{proof}

We can now return to the proof of Proposition \ref{prop:cons}. 

\begin{proof}[Proof of Proposition \ref{prop:cons}]
	
Since the support of $\mu_X$ is dense in $\mathcal{X}$, Lemma 4 gives $d_\mathcal{X} ( g \cdot X_{I(j)}, X_{J(j)} ) \overset{p}{\rightarrow} 0$. Now consider
\begin{align}
D_{I(j)J(j)}^{g_j} &= | g \star Y_{I(j)} - Y_{J(j)}| - V( g \cdot X_{I(j)}, X_{J(j)} ) \\
	&= | g \star f(X_{I(j)} ) + g \star \epsilon_{I(j)} - f(X_{J(j)}) - \epsilon_{J(j)} | - V( g \cdot X_{I(j)}, X_{J(j)} )
\end{align}
We have that $V( g \cdot X_{I(j)}, X_{J(j)} ) \rightarrow 0$ as the distance goes to $0$, and also 	$f(X_{J(j)}) - f( g \cdot X_{I(j)}) \overset{p}{\rightarrow} 0$. Thus with $X$ as a variable with the same law as each $X_i$ and $g$ with the same law as each $g_j$,
\begin{equation}
	D_{I(j)J(j)}^{g_j} \overset{D}{\rightarrow} | g \star (f(X) + \epsilon_\ell) - f(g \cdot X) - \epsilon_k  | = | \phi(g, X) + \eta |	
\end{equation}
where $\phi(g, X) = g \star f(X) - f ( g \cdot X )$ and $\eta = g \star \epsilon_l - \epsilon_k$. Definitionally, this means  
\begin{equation}
\PP( D_{I(j) J(j) }^{g_j} \geq t ) \rightarrow \PP( | \phi(g, X) + \eta | \geq t ) \end{equation}
Under $H_0$, $\phi(g,X) = 0$ almost surely, so this probability is given by $P_t^0 = \PP( |\eta| \geq t) = p_t$. Under $H_1$, $\phi$ must take non zero values with some positive probability. Thus we can use Lemma \ref{lem:sym_props} to say that 
\begin{equation}
	P_t^1 = \PP(  | \phi(g, X) + \eta | \geq t \mid H_1 ) > P_t^0
\end{equation}
Let $N$ be large enough that $\PP( D_{I(j) J(j) }^{g_j} \geq t ) > (P^1_t - P_0^t) / 2 > P_t^0$ for all $n > N$.  Now consider that $N_t^g \mid H_1 \sim \mathrm{Binom}( m, \PP( D_{I(j) J(j) }^{g_j} \geq t \mid H_1 )  )$, which is stochastically bounded from below by $A \sim \mathrm{Binom} (m, (P_t^1 - P_t^0)/2) $ for $n > N$. This gives, using lemma \ref{lem:binom_likelihood}, that
\begin{equation}
	p_{val} \mid H_1 = 1 - F_{ N_t^g \mid H_0 }( N_t^g \mid H_1 ) \leq 1 - F_{N_t^g \mid H_0 } ( A ) \overset{p}{\rightarrow} 0
\end{equation}
as required. 
\end{proof}

\newpage 
\section{Appendix: Simulations}
\label{app:sims}

\subsection{Effect of changing $V$ for the asymmetric variation test}
Simulations of the asymmetric variation test for varying tightness of the bound of $V$, shown in figure \ref{fig:effect_of_V}. The regression function $f_2$ is in all of the Lipschitz classes $\mathcal{F}(e^{-1}, 1) \subseteq \mathcal{F}(0.5, 1) \subseteq \mathcal{F}(1, 1) \subseteq \mathcal{F}(2, 1)$, so we can use the function $V(x,y) = L \| x - y \|_2$ from any of these classes. The simulations in figure \ref{fig:effect_of_V} (under the same set-up as above) show as expected that the tighter the bound the more powerful the test, though even weak bounds still converge an estimated power of $1$. 

The regression function is not in the classes $\mathcal{F}(e^{-3}, 1) \subseteq \mathcal{F}(e^{-2}, 1) \subseteq \mathcal{F}(e^{-1.2}, 1)$. Thus, using these values of $L$ are invalid. In figure \ref{sfig:emp_size_V}, we show the empirical size (i.e., rejection proportions under $H_0$) for these as well as the valid $L = e^{-1}$. We see that the size increases as $L$ decreases, but that the test is mildly robust to misspecification. In this case the $\mu_X$ volume of the region that breaks this bound is small. 
\begin{figure}[h]
	\centering
		\begin{subfigure}{.45\textwidth}
		\includegraphics[scale = 0.42]{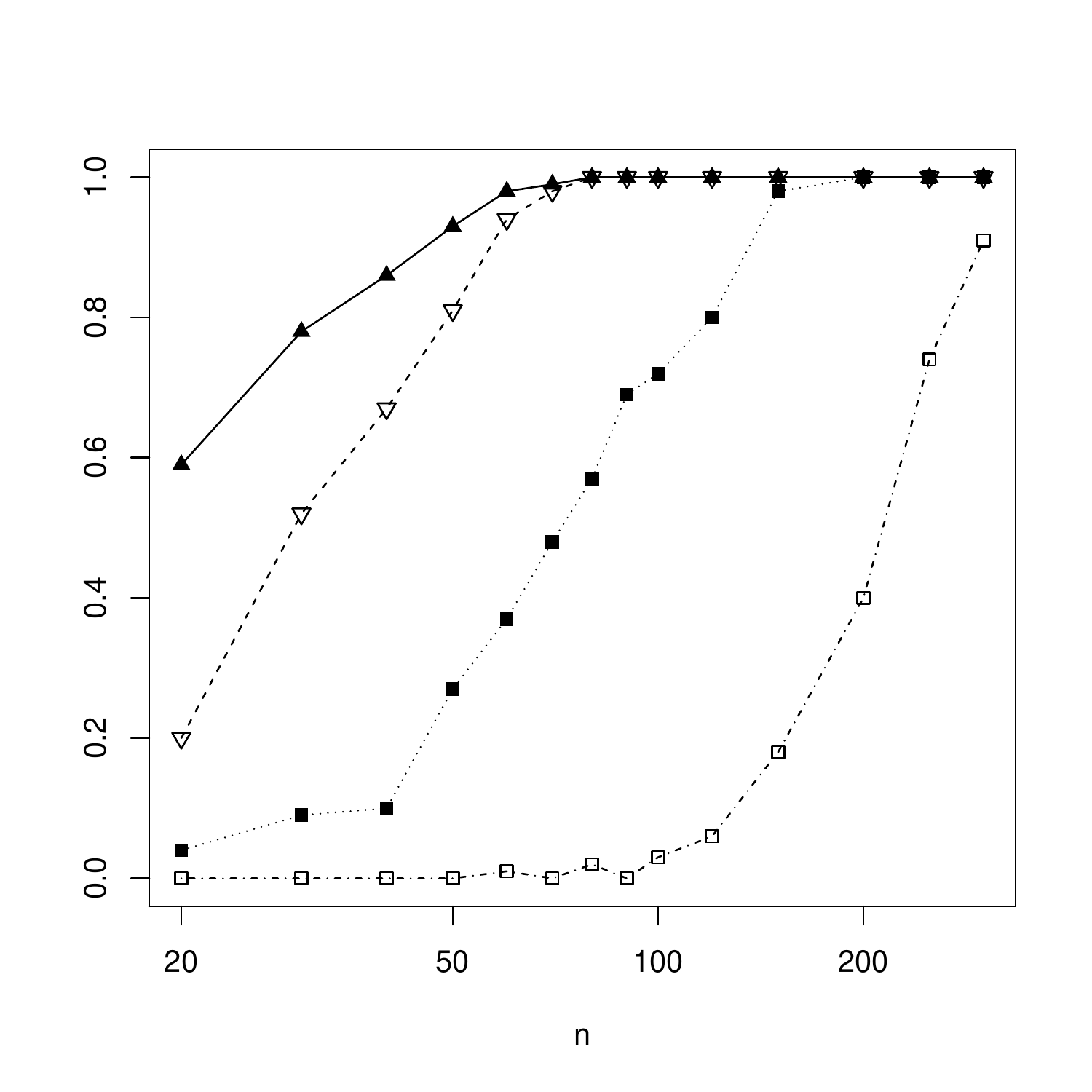}
		\caption{Estimated power of $H_1: f_2$ is $(G, \cdot)$-invariant.} 
	\end{subfigure}
	\begin{subfigure}{.45\textwidth}
		\includegraphics[scale=0.42]{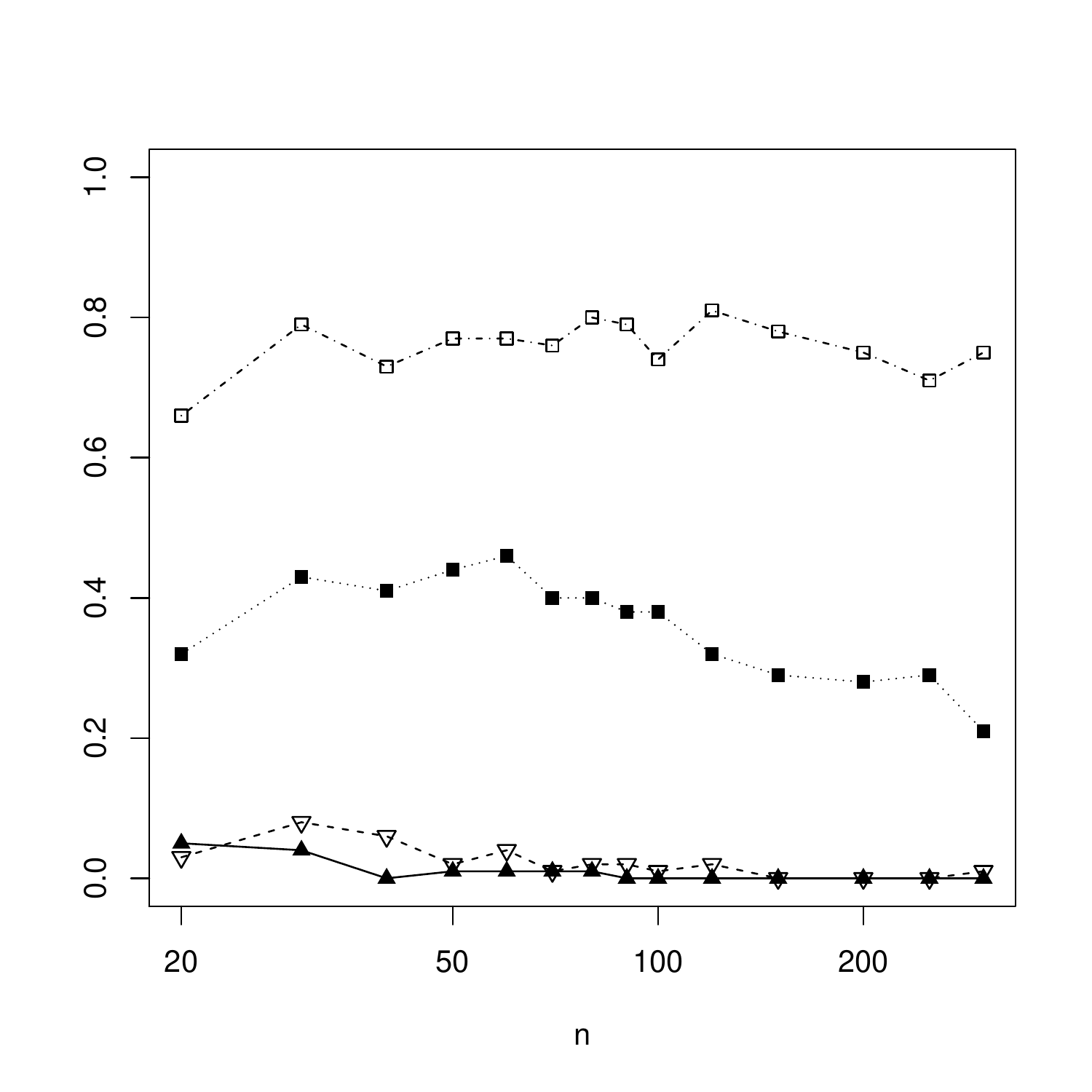}
		\caption{ Empirical Size for $H_0: f_2$ is $(G, \star)$-invariant. }
		\label{sfig:emp_size_V}
	\end{subfigure}
		\caption{Estimated power graphs of the Asymmetric Variation test for different levels of $V$. The line for $L = e^{-1}$ is is solid, for $L = 0.5$ dashed, for $L = 1$ dotted, and for $L = 2$ dot dashed. Empirical size graphs have a solid line for $L = e^{-1}$, dashed for $L = e^{-1.2}$, dotted for $L = e^{-2}$, and dot dashed for $L = e^{-3}$. }
	\label{fig:effect_of_V}
\end{figure}

\subsection{Effect of changing $q$ for the permutation variant}
\label{sapp:effect_q}
We show rejection proportions for the permutation variant for varying levels of $q \in \{ 0.5, 0.75, 0.9, 0.95, 1\}$. Tests with $q = 1$ are clearly more powerful, but are too liberal because of the bias of the finite sample of $\{ g \cdot X \}$. Tests with $q \leq 0.9$ are weaker than for $q = 0.95$, but without any significant reduction in the empirical size. 

\begin{figure}[h]
	\centering
	\begin{subfigure}{.45\textwidth}
		\includegraphics[scale=0.42]{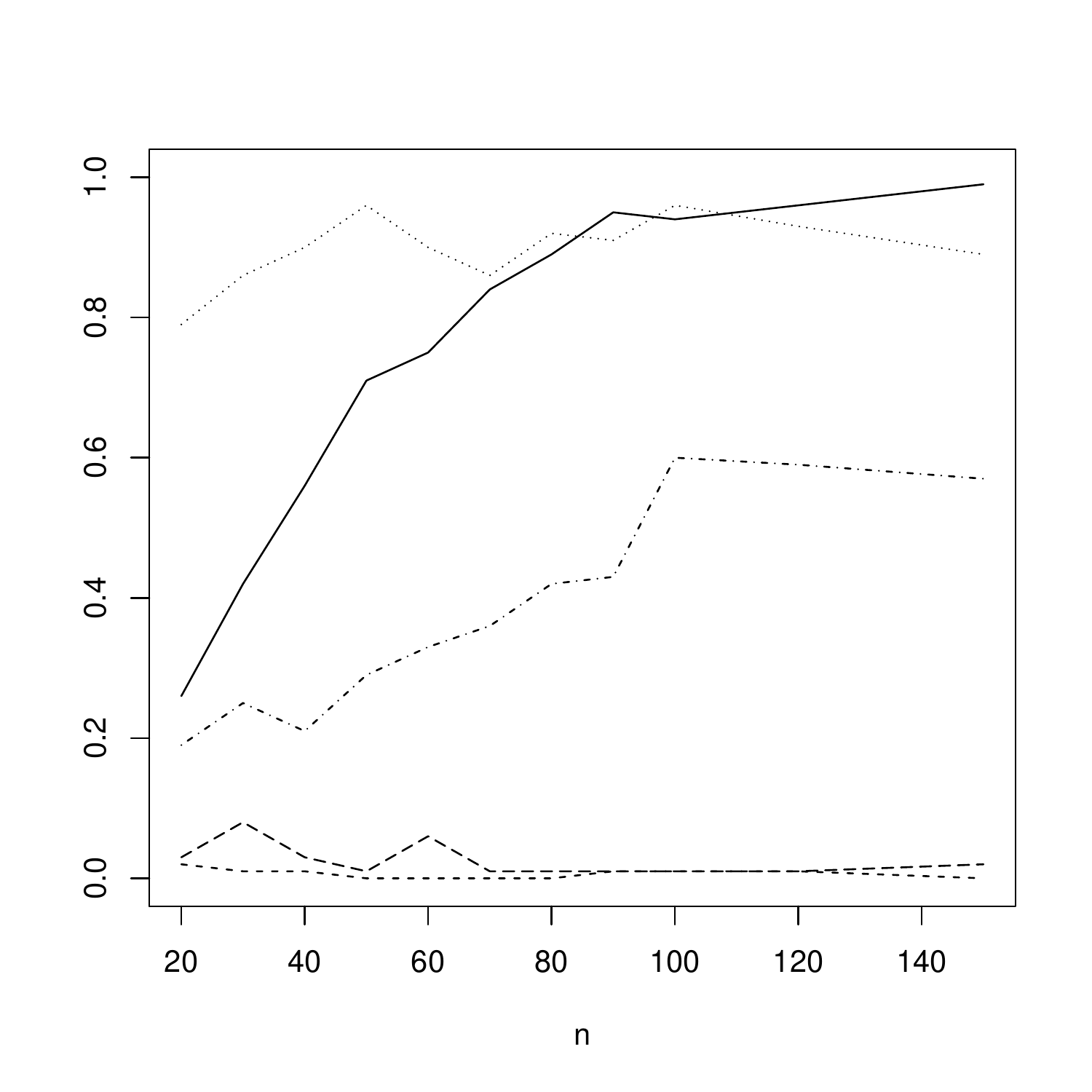}
		\caption{Estimated power of $H_0: f_2$ is $(G, \cdot)$-invariant.} 
	\end{subfigure}
	\begin{subfigure}{.45\textwidth}
		\includegraphics[scale=0.42]{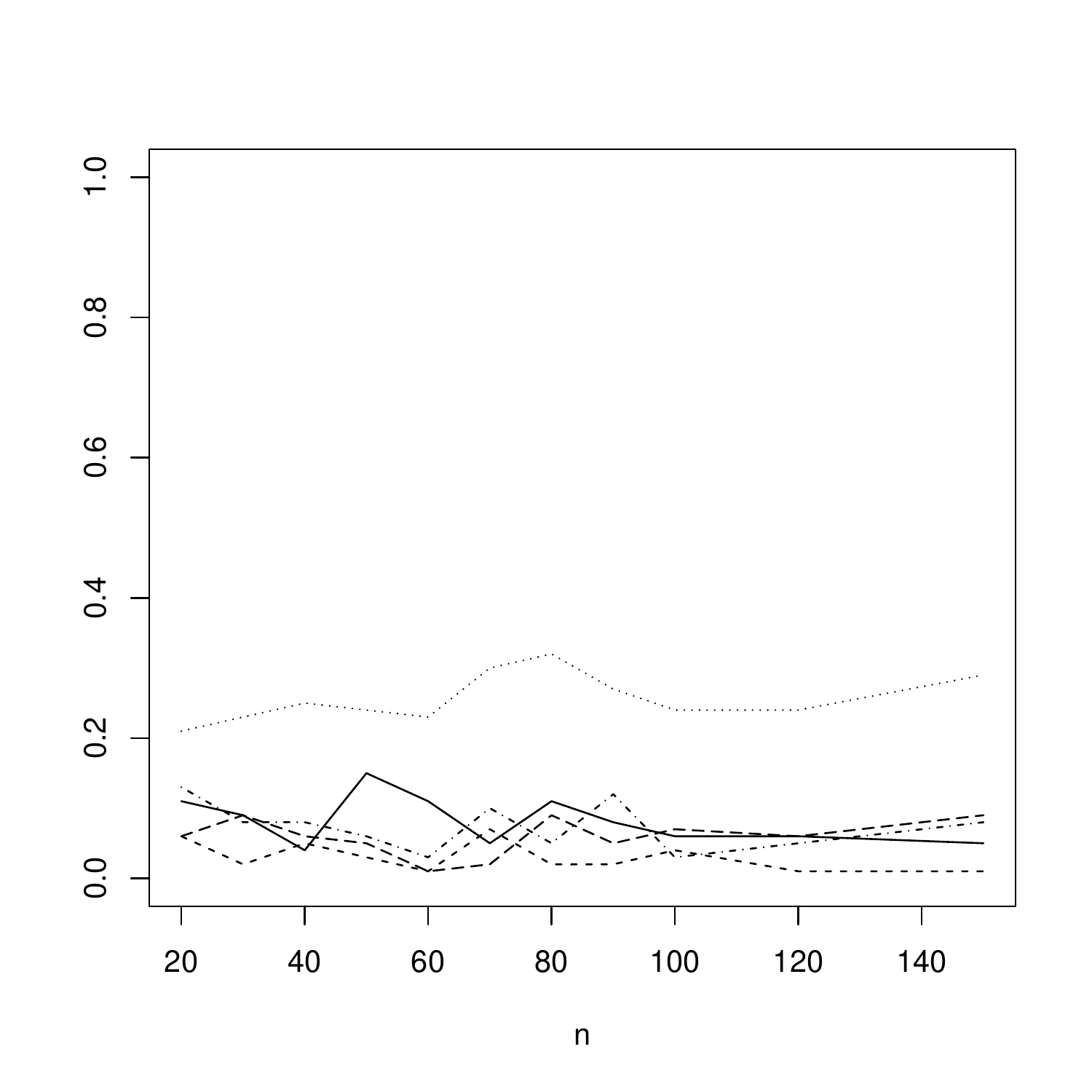}
		\caption{Empirical Size for $H_0: f_2$ is $(G, \star)$-invariant.}
	\end{subfigure}
	
		\caption{Rejection proportion graphs of the Permutation variant under difference choices of $q$. For $q = 0.5$, we have a dashed line, for $q = 0.75$ a long dashed line, for $q = 0.9$ a dot dashed line, for $q = 0.95$ a solid line, and for $q = 1$ a dotted line.  }
	\label{fig:effect_of_q}
\end{figure}

\subsection{Effect of asymmetry in figure \ref{fig:non_inv_error}} 
\label{sapp:fig_2_code}

Suppose that $X_i \overset{iid}{\sim} U( B_{\mathbb{R}^2}(0,4) )$ and $\epsilon_i \overset{iid}{\sim} U(-0.1,0.1)$. Set $Y_i = \exp( - \| X_i \|_2) + \epsilon_i$. Let $\hat{f}(x, \mathcal{D}) = \frac{ \sum_{i = 1}^n K_h(x,X_i) Y_i}{\sum_{i = 1}^n K_h(x, X_i ) }$ be a local constant estimator of $f(x) = \EE( Y_i \mid X_i = x) = e^{-\| x\|_2}$, with a rectangular kernel $K_h(x, X_i) = I( x \in B(X_i, h) )$ of bandwidth $h$. 

Let $G = \{ I, R_{\pi/2}, R_{\pi}, R_{3\pi / 2} \}$ act by rotations on $\RR^2$, which $f(x)$ is clearly invariant under. Let $K_G(x, X_i) = \tfrac{1}{4} \sum_{g \in G} K_h( g \cdot x, X_i)$ we the symmetrised kernel, and $\hat{f}_G$ the estimator using this symmetrised kernel. 

Suppose that we have an additional independent and identically distributed test dataset, also of size $n$. Figure \ref{fig:non_inv_error} tracks the mean squared error on the test dataset for 500 simulations at various values of $n$, plotting a black horizontal line for the error of $\hat{f}$ and a red line for $\hat{f}_G$.

\newpage

\subsection{Various other simulations}

Using the same notation as section \ref{ssec:sims}, we present several more simulations for various regression functions and choices of $n$ and $m$. These are from preliminary simulations and the code is not in the supplementary material.

\begin{eg}
Let $\mathcal{X} = \RR^2$, let $X \sim U( B(0,4) )$, and let $f_{sim}(x) = \exp( - \| x \|_2 )$ and $f_2(x) = \exp( - |x_1| )$. Let $Y^i \mid X \sim f_i(X) + \epsilon$ where $\epsilon \sim N(0, \sigma^2)$. Under the assumption that $f_i$ is $L$-Lipschitz, and with knowledge of $\sigma^2$, we obtain the following proportions of rejections of $H_0 : \delta_f = 0$ at significance level $\alpha = 0.05$ for 100 simulations:

 \begin{table}[h!]
 	\centering
 	\begin{tabular}{cc|cccccc|cccccc}
 	  & &\multicolumn{6}{c}{$n$ with $H_0: f = f_1$}	 & \multicolumn{6}{|c}{$n$ with $H_1: f = f_2$}	 \\
 	  & & 100 & 200 & 300 & 400 & 500 & 1000 & 100 & 200 & 300 & 400 & 500 & 1000 \\ \hline \hline
 	 \multirow{6}{1em}{$m$} & 100 & 0 & 0 & 0 & 0 & 0 & 0 & 0.08 & 0.45 & 0.69 & 0.70 & 0.87 & 0.97 \\ 
     & 200 & 0 & 0 & 0 & 0 & 0 & 0 & 0.17 & 0.61 & 0.81 & 0.96 & 0.97 & 1 \\
     & 300 & 0 & 0 & 0 & 0 & 0 & 0 & 0.24 & 0.77 & 0.95 & 0.99 & 1 & 1 \\
     & 400 & 0 & 0 & 0 & 0 & 0 & 0 & 0.24 & 0.82 & 0.98 & 1 & 1 & 1 \\
     & 500 & 0 & 0 & 0 & 0 & 0 & 0 & 0.31 & 0.89 & 1 & 1 & 1 & 1 \\
     & 1000 & 0 & 0 & 0 & 0 & 0 & 0 &0.41 & 0.93 & 1 & 1 & 1 & 1
 	\end{tabular}
 	\caption{Rejection proportions for $L = 1$ and $\sigma = 0.05$}
 \end{table}
 
 \begin{table}[h!]
 	\centering
 	\begin{tabular}{cc|cccccc|cccccc}
 	  & &\multicolumn{6}{c}{$n$ with $H_0: f = f_1$}	 & \multicolumn{6}{|c}{$n$ with $H_1: f = f_2$}	 \\
 	  & & 100 & 200 & 300 & 400 & 500 & 1000 & 100 & 200 & 300 & 400 & 500 & 1000 \\ \hline \hline
 	 \multirow{6}{1em}{$m$} & 100 & 0 & 0& 0& 0.01& 0.01& 0& 0.01& 0& 0.01& 0.01&  0.02&  0.02 \\
							& 200 & 0.01& 0.02& 0& 0.01& 0& 0.02& 0.01& 0& 0.01& 0.04&  0.02&  0.04 \\
							& 300 & 0.01& 0.01& 0.01& 0.02& 0.01& 0.01& 0.01& 0.01& 0.03& 0.01&  0.04&  0.06 \\
							& 400 & 0.04& 0& 0& 0.01& 0& 0.01& 0.02& 0.03& 0& 0.03&  0.04&  0.04 \\
							& 500 & 0& 0& 0.02& 0.03& 0.03& 0.02& 0.03& 0.04& 0.01& 0&  0.02&  0.08 \\
							&1000 & 0.01& 0.03& 0.04& 0.03& 0.03& 0.03& 0.02& 0.01& 0.02& 0.06&  0.08&  0.09 \\
 	\end{tabular}
 	\caption{Rejection proportions for $L = 1$ and $\sigma = 1$}
 \end{table} 
 
 \begin{table}[h!]
 	\centering
 	\begin{tabular}{cc|cccccc|cccccc}
 	  & &\multicolumn{6}{c}{$n$ with $H_0: f = f_1$}	 & \multicolumn{6}{|c}{$n$ with $H_1: f = f_2$}	 \\
 	  & & 100 & 200 & 300 & 400 & 500 & 1000 & 100 & 200 & 300 & 400 & 500 & 1000 \\ \hline \hline
 	 \multirow{6}{1em}{$m$} & 100 & 0   & 0 &   0 &   0  &  0 &   0 &   0 &0.01 &0.02 & 0.11 & 0.12 & 0.49 \\
& 200&    0  &  0   & 0 &   0&    0&    0&    0& 0 &0.04&  0.10&  0.24&  0.79 \\
& 300&   0  &  0   & 0   & 0  &  0  &  0  &  0 &0 &0.02  &0.20 & 0.31 & 0.96 \\
& 400&   0  &  0   & 0   & 0   & 0   & 0 &   0 &0 &0.01  &0.19 & 0.32  &0.98 \\
& 500&    0 &   0   & 0  &  0   & 0   & 0 &   0 &0& 0 & 0.18 & 0.44  &1 \\
& 1000&    0 &   0 &   0 &   0   & 0   & 0 &   0& 0& 0.03 & 0.22 & 0.64  &1 \\
 	\end{tabular}
 	\caption{Rejection proportions for $L = 2$ and $\sigma = 0.05$}
 \end{table} 
 
\end{eg}

\newpage

\begin{eg}
Let $\mathcal{X} = \RR^2$, let $X \sim U( B(0,4) )$, and let $f_3(x) = \| x \|_2$ and $f_4(x) = |x_1| $. Let $Y^i \mid X \sim f_i(X) + \epsilon$ where $\epsilon \sim N(0, \sigma^2)$. Note that $\delta_{f_4} \approx 0.72$. Under the assumption that $f_i$ is $L$-Lipschitz, and with knowledge of $\sigma^2$, we obtain the following proportions of rejections of $H_0 : \delta_f = 0$ at significance level $\alpha = 0.05$ for 100 simulations:

 \begin{table}[h!]
 	\centering
 	\begin{tabular}{cc|cccccc|cccccc}
 	  & &\multicolumn{6}{c}{$n$ with $H_0: f = f_3$}	 & \multicolumn{6}{|c}{$n$ with $H_1: f = f_4$}	 \\
 	  & & 100 & 200 & 300 & 400 & 500 & 1000 & 100 & 200 & 300 & 400 & 500 & 1000 \\ \hline \hline
 	 \multirow{6}{1em}{$m$} & 100 & 0 & 0 & 0 & 0 & 0 & 0 & 1 & 1 & 1 & 1 & 1 & 1 \\ 
     & 200 & 0 & 0 & 0 & 0 & 0 & 0 & 1 & 1 & 1 & 1 & 1 & 1 \\
     & 300 & 0 & 0 & 0 & 0 & 0 & 0 & 1 & 1 & 1 & 1 & 1 & 1 \\
     & 400 & 0 & 0 & 0 & 0 & 0 & 0 & 1 & 1 & 1 & 1 & 1 & 1 \\
     & 500 & 0 & 0 & 0 & 0 & 0 & 0 & 1 & 1 & 1 & 1 & 1 & 1 \\
     & 1000 & 0 & 0 & 0 & 0 & 0 & 0 & 1 & 1 & 1 & 1 & 1 & 1
 	\end{tabular}
 	 \caption{Rejection proportions for $L =1$ and $\sigma = 0.05$}
 \end{table}

  \begin{table}[h!]
 	\centering
 	\begin{tabular}{cc|cccccc|cccccc}
 	  & &\multicolumn{6}{c}{$n$ with $H_0: f = f_3$}	 & \multicolumn{6}{|c}{$n$ with $H_1: f = f_4$}	 \\
 	  & & 100 & 200 & 300 & 400 & 500 & 1000 & 100 & 200 & 300 & 400 & 500 & 1000 \\ \hline \hline
 	 \multirow{6}{1em}{$m$} & 100 & 0.01& 0.01& 0 & 0.02& 0.02& 0.03& 0.31& 0.51& 0.64&  0.61&  0.61&  0.69 \\
							& 200 & 0.02& 0.01& 0.02& 0& 0.02& 0& 0.62& 0.74& 0.79&  0.88&  0.93&  0.97\\
							& 300 & 0.02& 0.02& 0.02& 0.01& 0.01& 0.01& 0.61& 0.78& 0.92&  0.94&  0.97&  0.99\\
							& 400 & 0.01& 0.04& 0& 0& 0.02& 0.03& 0.67& 0.90& 0.95&  0.99&  0.98&  1\\
							& 500 & 0.02& 0.04& 0.03& 0.01& 0.04& 0& 0.68& 0.95& 0.96&  0.98&  1&  1\\
						   & 1000 & 0.03& 0.00& 0.03& 0.00& 0.02& 0.08& 0.85& 0.96& 1&  1&  1&  1
 	\end{tabular}
 	 \caption{Rejection proportions for $L= 1 $ and $\sigma = 1$}
 \end{table}
 
  \begin{table}[h!]
 	\centering
 	\begin{tabular}{cc|cccccc|cccccc}
 	  & &\multicolumn{6}{c}{$n$ with $H_0: f = f_3$}	 & \multicolumn{6}{|c}{$n$ with $H_1: f = f_4$}	 \\
 	  & & 100 & 200 & 300 & 400 & 500 & 1000 & 100 & 200 & 300 & 400 & 500 & 1000 \\ \hline \hline
 	 \multirow{6}{1em}{$m$} & 100 & 0 & 0 & 0 & 0 & 0 & 0 & 1 & 1 & 1 & 1 & 1 & 1 \\ 
     & 200 & 0 & 0 & 0 & 0 & 0 & 0 & 1 & 1 & 1 & 1 & 1 & 1 \\
     & 300 & 0 & 0 & 0 & 0 & 0 & 0 & 1 & 1 & 1 & 1 & 1 & 1 \\
     & 400 & 0 & 0 & 0 & 0 & 0 & 0 & 1 & 1 & 1 & 1 & 1 & 1 \\
     & 500 & 0 & 0 & 0 & 0 & 0 & 0 & 1 & 1 & 1 & 1 & 1 & 1 \\
     & 1000 & 0 & 0 & 0 & 0 & 0 & 0 & 1 & 1 & 1 & 1 & 1 & 1
 	\end{tabular}
 	 	\caption{Rejection proportions for $L = 2$ and $\sigma = 0.05$}
 \end{table}
 
\end{eg}

\end{document}